\documentclass[sigconf, nonacm]{acmart}

\usepackage{algorithm}
\usepackage{algpseudocode}
\usepackage{amsmath,amsfonts}
\usepackage{hyperref}
\usepackage{mathtools}
\usepackage{pdflscape}
\usepackage{afterpage}

\makeatletter
\newcommand{\LineLabel}[1]{%
  \hypertarget{#1}{\phantomsection}\label{#1}%
}
\makeatother
\usepackage{amsthm}
\newtheorem{theorem}{Theorem}[section]

\newtheorem{definition}{Definition}[section]

\usepackage{caption}
\usepackage{subcaption}

\usepackage[nameinlink,capitalize]{cleveref}

\usepackage{tikz}
\usepackage{xcolor}
\usetikzlibrary{fit,calc}

\newcommand{\tikzmk}[1]{\tikz[remember picture,overlay,baseline=(#1.base)]\node (#1) {\strut};}

\newcommand{\boxit}[1]{%
    \begin{tikzpicture}[remember picture,overlay]

        \coordinate (topleft) at ($(A.base)+(0, 2.1ex)$); 
    
        \coordinate (bottomright) at ($(B.base)+(0, -0.7ex)$);
        
        \node [fill=#1, fill opacity=0.2, inner sep=0pt, rectangle, 
               fit=(topleft) (bottomright)] {};
    \end{tikzpicture}%
}
\newcommand{\strikeline}[1]{%
    \begin{tikzpicture}[remember picture,overlay]
        \draw[red, thick, opacity=0.6] 
            ($(A.west)+(0, 0.2ex)$) -- ($(B.east)+(0, 0.4ex)$);
    \end{tikzpicture}%
}

\usepackage[table]{xcolor} %
\usepackage{booktabs}      %
\usepackage{multirow}      %

\newcommand{\ourPE}{PE-means}
\newcommand{\augpe}{Aug-PE}

\newcommand{\bra}[1]{\left( #1 \right)}
\newcommand{\brb}[1]{\left[ #1 \right]}

\newcommand{\brc}[1]{\left\{ #1 \right\}}

\newcommand{\dpvotingname}{DP Nearest Neighbors Histogram}
\newcommand{\dpvotingfunctionname}{\textsf{DP\_NN\_HISTOGRAM}}
\newcommand{\dpvotingfunction}[1]{\dpvotingfunctionname{}\bra{#1}}

\newcommand{\distancefunctionname}{d}
\newcommand{\distancefunction}[1]{\distancefunctionname{}\bra{#1}}

\newcommand{\numiterations}{T}
\newcommand{\numgensamples}{N_\syn}

\newcommand{\noisemultiplier}{\sigma}

\newcommand{\normaldistribution}[2]{\calN\bra{#1,#2}}

\newcommand{\priv}{\mathrm{priv}}
\newcommand{\syn}{\mathrm{syn}}

\newcommand{\calN}{\mathcal{N}}

\newcommand{\Dataset}{\ensuremath{D}}
\newcommand{\Ddim}{\ensuremath{d}}
\newcommand{\Rdim}{\ensuremath{p}}
\newcommand{\DatasetSize}{\ensuremath{N}}
\newcommand{\radius}{\ensuremath{R}}
\newcommand{\mutationScale}{\ensuremath{\gamma}}
\newcommand{\numVariations}{\ensuremath{L}}

\newcommand{\google}{Google-LSH}
\newcommand{\fastlloyd}{FastLloyd}
\newcommand{\diffpriv}{DP-Lib}
\newcommand{\balcan}{Icml17}

\begin{document}
\title{PE-means: Improved Differentially Private $k$-means Clustering through Private Evolution}

\author{Thomas Humphries}
\affiliation{%
  \institution{University of Waterloo}
}
\email{thomas.humphries@uwaterloo.ca}

\author{Zinan Lin}
\affiliation{%
  \institution{Microsoft Research}
}
\email{zinanlin@microsoft.com}

\author{Sergey Yekhanin}
\affiliation{%
  \institution{Microsoft Research}
}
\email{yekhanin@microsoft.com}

\begin{abstract}
    We study the problem of differentially private (DP) $k$-means clustering in Euclidean space. 
    Previous solutions rely on summing the private data directly, which induces a sensitivity proportional to the domain.
    We introduce PE-means, an extension of the private evolution (PE) algorithm (an increasingly popular method for synthetic data generation), to the problem of $k$-means clustering.
    The key advantage of PE is that it only computes a private histogram with constant sensitivity to guide the evolution.
    Our adaptation of PE includes new evolutionary operators for clustering, as well as other algorithmic improvements of independent interest. 
    Overall, PE-means achieves an average improvement of 26\% in clustering loss over state-of-the-art baselines such as Google's LSH-based algorithm and DP-Lloyd variants.
\end{abstract}

\maketitle

\section{Introduction}
The $k$-means clustering algorithm is a foundational tool in data science with many influential applications such as recommendation systems and healthcare data analysis~\cite{mcsherry2009,ghosal2020short}.
The objective is to group the data into $k$ groups with similar features, allowing practitioners to interpret complex datasets.
The challenge is that applying standard $k$-means algorithms to sensitive data risks the privacy of participants in the dataset.
In the worst case, the $k$-means algorithm can publish the exact data record of an individual who is an outlier in the set.
To address this, there has been a large body of work that considers the problem of differentially private $k$-means~\cite{su_clustering, diaa2025fastlloyd, google_clustering, blum2005, icml17_clustering, stemmer2018, mcsherry2009, Feldman2009, Ghazi2020}.
Differential privacy (DP) is a popular privacy definition that ensures a change in a single user's input does not have a significant effect on the output of the mechanism~\cite{dwork-dp}.
By adding calibrated randomization to the clustering process, existing approaches retain the benefits of $k$-means clustering while providing a formal privacy guarantee.

The challenge with DP is balancing the inherent trade-off between privacy and utility.
Existing state-of-the-art approaches for DP $k$-means rely on summing the private data directly at some point in the process~\cite{google_clustering, su_clustering, icml17_clustering}.
The sensitivity of (or largest impact any one user can have on) a sum query is determined by the addition (or removal) of a point on the boundary of the domain.
This means that to satisfy DP, noise must be added proportionally to the entire domain, significantly degrading the quality of the clustering.
In recent work, an approach called FastLloyd managed to reduce this sensitivity from the whole domain to a smaller radius around each cluster through relative cluster updates~\cite{diaa2025fastlloyd}.
However, the radius is still proportional to the data dimension, and FastLloyd introduces an additional error term due to the relative updates.

Private evolution (PE) is an increasingly popular technique for generating various modalities of synthetic data~\cite{pe, augpe, pe_three,tran2026differentially}.
PE evolves a population of synthetic data using only inference API access to foundation models that generate and perturb the data in an evolutionary algorithm.
A key component in PE's success is that it only uses the private data in a nearest neighbour voting scheme to choose the best samples generated so far.
Specifically, each private data point only contributes a single vote to the algorithm in each iteration.
Thus, adding or removing any private data point affects the output by a constant sensitivity of $1$, regardless of the data domain.
This gives PE a significant advantage over gradient-based techniques, which must add noise proportionally to the gradient domain.
What makes PE truly impressive is that, despite getting much less signal from the private data, it still effectively traverses large and complex data domains by leveraging the robust optimization characteristics of evolution.

In this work, we introduce \ourPE, an extension of the PE framework to the problem of $k$-means clustering. 
In this application, the population of PE contains randomly generated centroids that are iteratively evolved towards the centroids of the private data.
Our design replaces PE's API calls to foundation models with a lightweight clustering initialization and a Lévy flight-based mutation operator.
In addition to adapting PE to clustering, we also make several improvements to the algorithm that are of independent interest.
Specifically, we address a critical issue with PE's selection technique where votes can be split between the best candidates, causing no representative from certain areas to be selected.
We also reduce the impact of DP noise through better post-processing of the vote histogram and a mechanism to adaptively adjust the signal-to-noise ratio at no cost to privacy.
Our experimental evaluation shows \ourPE~outperforms state-of-the-art baselines over numerous real and synthetic datasets.
Furthermore, we propose HD\ourPE, a variant of \ourPE~that uses a similar dimensionality reduction to Balcan et al.'s~\cite{icml17_clustering}, allowing \ourPE~to scale to higher dimensions.
Overall, we observe up to 91\% improvement in clustering loss over the state-of-the-art, with an average improvement of $26\%$ across all the datasets evaluated.

The rest of the paper is organized as follows. We first give the necessary background in Section~\ref{sec:background}. Section~\ref{sec:method} details our design through a step-by-step treatment of the PE algorithm. In Section~\ref{sec:experiments}, we provide all experimental evidence including a case study on vote-splitting, the tuning of hyperparameters, and a utility evaluation of \ourPE~and HD\ourPE. Finally, Section~\ref{sec:related_work} surveys the relevant literature before concluding in Section~\ref{sec:conclusion}.

\section{Background}\label{sec:background}
\subsection{$k$-Means Clustering}\label{sec:kmeans}
Given a dataset $\Dataset$ of size $\DatasetSize$ and dimension $\Ddim$, the $k$-means problem aims to partition the dataset into $k$ groups (clusters) $C=\{C_1, \dots, C_k\}$, with each $C_i\subset\Dataset$, such that the following objective is minimized,
\begin{equation}\label{eqn:k_means_objective}
\arg\min_{C} \sum_{i=1}^{k} \sum_{x \in C_i} \|x - \mu_i\|^2
\end{equation}
where $\mu_i$ is the centroid of cluster $i$ ($\mu_i = \frac{1}{|C_i|} \sum_{x \in C_i} x$).
In this work, we assume the size and dimension of the dataset are public information, but we require any other information published from the dataset to satisfy differential privacy.

\subsection{Differential Privacy}
Differential privacy (DP)~\cite{dwork-dp} adds randomness to the computation of aggregate statistics to protect the privacy of individuals.
DP guarantees that an algorithm's output is approximately the same, regardless of the participation of any one user.
More formally, differential privacy can be defined as follows.
\begin{definition}[Differential Privacy]\label{def.adp}
    A randomized algorithm $M:\mathcal{D} \mapsto \mathbb{R}$ is $(\epsilon,\delta)$-DP, if for any pair of neighbouring datasets $\Dataset,\Dataset' \in \mathcal{D}$, and for any $S \subseteq \mathbb{R}$ we have
    \begin{equation}
     \Pr[M(\Dataset)\in S] \leq e^{\epsilon} \Pr[M(\Dataset')\in S] +\delta .
    \end{equation}
\end{definition}
The privacy parameter $\epsilon$ defines how similar the output distributions must be, and $\delta$ allows a small chance of failure in the definition.
We use the unbounded neighbouring definition, where datasets are neighbours if $|\Dataset\backslash \Dataset'\cup \Dataset'\backslash \Dataset| = 1$ (we allow for the addition or removal of a single data point). 
Arbitrary computations can be carried out on the output of a DP mechanism without affecting privacy due to the post-processing lemma~\cite{dwork14}.
Finally, DP is composed naturally with multiple runs of a mechanism.
If we apply differentially private mechanisms sequentially, the privacy parameter composes through summation or more advanced methods~\cite{dwork14}.

\begin{definition}[Sensitivity]\label{def:sensitivity}
    Let $f:\mathcal{D}\mapsto \mathbb{R}^k$. If $\ \mathbb{D}$ is a distance metric between elements of $\ \mathbb{R}^k$, then the $\mathbb{D}$-sensitivity of $f$ is
    \begin{equation}
        \Delta^{(f)}=\max_{(D,D')} \mathbb{D}(f(\Dataset), f(\Dataset')),
    \end{equation}
    where $(\Dataset,\Dataset')$ are pairs of neighbouring datasets.
\end{definition}
We primarily use the $\ell_2$ norm as the distance metric $\mathbb{D}$. To satisfy DP, we use Gaussian noise in this work and analyze it with Gaussian Differential Privacy (GDP)~\cite{dong22_gdp}.

\begin{definition}[GDP~\cite{dong22_gdp}] \label{def:GDP}
    A mechanism $M$ is said to satisfy $\theta$-Gaussian Differential Privacy ($\theta$-GDP) if
    \[
    \mathcal{T}\big(M(D),M(D')\big) \ge G_\theta
    \]
    for all neighbouring datasets $D$ and $D'$, where $\mathcal{T}$ is a trade-off function measuring the difficulty for attackers in identifying the presence of an individual data point and $G_\theta=\mathcal{T}\big(\mathcal{N}(0,1),\mathcal{N}(\theta,1)\big)$ (see Dong et al.~\cite{dong22_gdp} for specifics of the definition).
\end{definition}
 Naturally, the Gaussian Mechanism satisfies GDP.
\begin{theorem}(Gaussian Mechanism for GDP~\cite{dong22_gdp})\label{thm:g_mech}
    Define the Gaussian mechanism that operates on a statistic $f$ as $M(D) = f(D) + \eta$, where $\eta \sim \mathcal{N}(0, (\Delta^{(f)})/\theta)$. Then, $M$ is $\theta$-GDP.
\end{theorem}
We consider the composition of GDP over multiple runs of a mechanism.
\begin{theorem}[GDP Composition~\cite{dong22_gdp}]\label{thm:gdp_comp}
    The $n$-fold composition of $\theta_i$-GDP mechanisms is $\sqrt{\theta_1^2+\cdots+\theta_n^2}$-GDP.
\end{theorem}
We also have a way to convert between GDP and DP.
\begin{theorem}[GDP to DP~\cite{dong22_gdp, balle18_gaussian}]\label{thm:GDPtoDP}
    A mechanism is $\theta$-GDP if and only if it is $\big(\epsilon,\delta(\epsilon)\big)$-DP for all $\epsilon \geq 0$, where
    \[
    \delta(\epsilon)= \Phi\Big( -\frac{\varepsilon}{\theta} +\frac{\theta}{2} \Big)-
    e^{\epsilon}\Phi\Big(- \frac{\varepsilon}{\theta} - \frac{\theta}{2} \Big).
    \]
\end{theorem}
In practice, we follow the implementation of PE\footnote{\url{https://github.com/microsoft/DPSDA}} to solve this function for $\theta$, given a desired $\epsilon$ and $\delta$.

\section{Methodology}\label{sec:method}

\begin{algorithm}[t]
    \caption{\colorbox{gray!20}{\textbf{\ourPE}}~based on Private Evolution (PE)~\cite{pe}}
    \label{alg:mod_PE}
    \begin{algorithmic}[1]
        \renewcommand{\algorithmicrequire}{\textbf{Input:}}
        \renewcommand{\algorithmicensure}{\textbf{Output:}}
        \Require \begin{tabular}[t]{@{}l@{}}
            Private samples: $\Dataset=\brc{x_i}_{i=1}^{\DatasetSize}\subset \mathbb{R}^\Ddim$ \\
            Number of iterations: $\numiterations$ \\
            Number of generated samples: $\numgensamples \colorbox{gray!20}{= k}$ \\
            Number of initial variations: $\numVariations$\\
            Noise multiplier for the \dpvotingname{}: $\noisemultiplier$ 
        \end{tabular}
        \Ensure DP cluster centres $S'_T = \{\mu_1, \dots, \mu_k\} \subset \mathbb{R}^d$
        \Statex %
        \State $S_0 \leftarrow $\colorbox{gray!20}{\textsf{RANDOM\_API}}$(\numgensamples * (\numVariations+1))$ \LineLabel{line:random_api}
        \For{$t \leftarrow 1, \ldots, \numiterations$}
            \State $hist_t \leftarrow \dpvotingfunction{\Dataset, S_{t-1}, \noisemultiplier}$ \Comment{\cref{alg:voting}}\LineLabel{line:dp_voting}
            \State \tikzmk{A}\textcolor{gray}{$S_{t}' \leftarrow$ \textbf{select top-}$N_\syn$ \textbf{samples} from $S_{t-1}$ ranked by $\text{hist}_t$} \tikzmk{B} \strikeline{red}\LineLabel{line:top_k_select}
            \State \tikzmk{A}$S_{t}' \leftarrow$ \textsf{WEIGHTED\_K\_MEANS}($S_{t-1}$, k=$\numgensamples$, weight=$hist_t$)\tikzmk{B} \boxit{gray}\LineLabel{line:k_means_select}
            \State $S_{t} \leftarrow S_{t}'$ \Comment{Save best samples}\LineLabel{line:elitisim}
            \If{\tikzmk{A}sum($hist_t^2)/k(\numVariations+1) \noisemultiplier^2 < 1.0$\tikzmk{B} \boxit{gray}}\LineLabel{line:check_signal_noise}
            \State \tikzmk{A} $\numVariations \gets \numVariations/2 $ \tikzmk{B} \boxit{gray}\LineLabel{line:reduce_L}
            \EndIf
            \For{$i \gets 1, \dots, L$}
            \State $S_{t} \leftarrow S_{t} \cup $\colorbox{gray!20}{\textsf{VARIATION\_API}}$({S_t'})$ \LineLabel{line:variation}
            \EndFor
        \EndFor
    \State  \Return $S'_{T}$\LineLabel{line:return_centroids}
    \end{algorithmic}
\end{algorithm}

Private Evolution (PE) is an emerging private synthetic data generation technique that evolves a set of synthetic samples using only inference API access to large machine learning models~\cite{pe, augpe, pe_three,tran2026differentially}.
At a high level, PE starts from randomly generated synthetic data (using the \textsf{RANDOM\_API}) and iteratively evolves toward high-quality synthetic data by creating new variations (using the \textsf{VARIATION\_API}) and ranking the synthetic data that is closest to the private data.
We build upon the text variant of PE called \augpe~\cite{augpe}.
In Algorithm~\ref{alg:mod_PE}, we outline the baseline version of PE.
Algorithm~\ref{alg:mod_PE} also illustrates our approach, \ourPE, highlighting our modifications to both improve PE and adapt it to solve $k$-means clustering.
In this section, we detail the design of \ourPE~with respect to each part of the algorithm. Following this, we present our high-dimensional variant of \ourPE~in Section~\ref{sec:hd_pe} and prove the privacy of both algorithms in Section~\ref{sec:privacy}.

\subsection{Random API}\label{sec:random}
The first step in PE is to generate an initial population, independently of the private data, using the \textsf{RANDOM\_API} (Line~\ref{line:random_api}). 
In the earliest variants of PE \cite{pe,augpe}, this is done by querying a pre-trained foundation model (using information such as the label, which is assumed to be public).
To solve clustering with PE, rather than creating a population of synthetic data samples, we create a population of centroids to evolve.
We assume the only context available to generate the initial population is the shape of the data domain.
For simplicity, we assume the domain is a hyper-sphere and use the same $\ell_2$ bound $\radius$ on its radius that other approaches need to bound sensitivity~\cite{su_clustering, google_clustering, icml17_clustering, diaa2025fastlloyd}.
However, in our case, the bound being too tight or loose can at most affect convergence speed and has no effect on the amount of noise added.
A naive \textsf{RANDOM\_API} for clustering would simply generate uniformly random data within the hyper-sphere.
However, random points may group together, which reduces our coverage of the domain.
Thus, we instead follow the work of Su et al.~\cite{su_clustering} and use the sphere packing method.
Su et al.'s method starts with a given radius $a$ (typically half the radius of the domain), and iteratively adds random samples that are at least $a$ away from the boundary and $2a$ from the samples added so far. 
If, after a fixed number of retries, a random sample cannot be found to meet these conditions, $a$ is halved and the algorithm continues until $k$ samples are found.
We adapt Su et al.'s algorithm to consider an $\ell_2$ bounded hyper-sphere rather than an $\ell_\infty$ bounded hypercube.
The rest of the algorithm remains the same.

\subsection{Selection Step}\label{sec:selection}
\begin{algorithm}[t]
    \caption{\colorbox{gray!20}{\textbf{Modified}} \dpvotingfunctionname{}}
    \label{alg:voting}
    \begin{algorithmic}[1]
        \renewcommand{\algorithmicrequire}{\textbf{Input:}}
        \renewcommand{\algorithmicensure}{\textbf{Output:}}
        
        \Require \begin{tabular}[t]{@{}l@{}}
            Private samples: $\Dataset$ \\
            Generated samples: $S=\brc{z_i}_{i=1}^{n}$ \\
            Noise multiplier: $\noisemultiplier$ \\
            Distance function: $\distancefunction{\cdot,\cdot}$
        \end{tabular}
        
        \Ensure DP nearest neighbours histogram on $S$
        
        \Statex 
        
        \State $histogram \leftarrow [0, \dots, 0]$
        \For{$x_{\priv} \in \Dataset$}
            \State $i = \arg\min_{j \in \brb{n}} \distancefunction{x_{\priv}, z_j}$\LineLabel{line:nn_compute}
            \State $histogram[i] \leftarrow histogram[i] + 1$\LineLabel{line:vote}
        \EndFor
        
        \State $histogram \leftarrow histogram + \normaldistribution{0}{\noisemultiplier I_n}$ \LineLabel{line:noise_votes}
        
        \State \tikzmk{A}$histogram \gets \text{sort\_descending}(histogram)$ \LineLabel{line:sort_votes}
        
        \State$k \gets \min \brc{ \kappa \mid \sum_{j=1}^{\kappa} histogram[j] > \DatasetSize }$ \LineLabel{line:choose_large_votes}
        
        \State $histogram[j] \gets 0$ \textbf{for all} $j > k$  \Comment{Zero out noisy values} \tikzmk{B} \boxit{gray}\LineLabel{line:zero_small_votes}
        
        \State \Return $histogram$
    \end{algorithmic}
\end{algorithm}

The next step in PE is to select the best candidates from the population (either from the random initialization or previous iteration).
First, the selection computes a private histogram of votes, then selects candidates based on these votes.
We detail our modifications to both parts.

\paragraph{Histogram Computation}
To utilize the private data in the selection, we compute the \textsf{DP\_NN\_HISTOGRAM} (Line~\ref{line:dp_voting}) described in Algorithm~\ref{alg:voting}.
The idea is that all private samples vote for their nearest neighbour in the population (Lines~\ref{line:nn_compute} and \ref{line:vote}).
This creates a histogram of votes that is easy to privatize using Gaussian noise in Line~\ref{line:noise_votes}.
However, depending on the size of the population and progression of the evolution, this histogram can be quite sparse.
Some of the previous PE versions \cite{pe,pe_three,tran2026differentially} post-processed the histogram by zeroing out any noisy votes below a certain threshold.
The challenge is how to choose the optimal threshold without spending privacy budget.
We propose a modification based on the constraint that, before adding noise, the histogram of votes should be non-negative and sum to the size of the dataset (which we assume to be public).
The maximum likelihood estimation (MLE) of the true votes given these constraints is found by projecting the noisy votes onto the set $\{x\in \mathbb{R} | x_i\geq0, \sum_i x_i=\DatasetSize\}$.
Duchi et al.~\cite{Duchi08} give an efficient algorithm for this projection by computing a subset of bins with the largest noisy counts such that after scaling and thresholding they sum to exactly $\DatasetSize$.
For clustering, we do not need the sum constraint to be exact; a more critical issue is points far from the private data getting assigned a few false votes due to noise and pulling the centroids in the wrong direction.
Thus, we give a simple approximation of the MLE algorithm that computes the minimum number of highest-count buckets that exceed the sum constraint (Line~\ref{line:choose_large_votes}) and zeros out the remaining buckets (Line~\ref{line:zero_small_votes}), without any additional scaling.
Intuitively, we assume histogram buckets with large counts are likely to have dominated the noise.

\paragraph{Selecting Candidates}
After computing the private histogram, the next evolutionary task is to choose the best candidates from the population to survive to the next iteration based on this histogram.
\augpe~takes the approach of ranking candidates based on the number of votes they received in the histogram (Line~\ref{line:top_k_select}).
The drawback to this approach is that choosing the most popular points does not always yield optimal clustering coverage due to what we call vote splitting.
Consider a scenario where PE generates a point that is the nearest neighbour of many private dataset points.
In the next iteration, PE may generate numerous new samples close to this point, causing the votes to be split between these new samples.
Then, since all new points have a low count, an entire area of feature space may not be selected. %
We illustrate an example of this for clustering in Section~\ref{sec:vote_split}.
We argue that instead of just choosing points with the highest votes, we also need to choose one or two solutions from areas with highly split votes.
Our key observation is that standard $k$-means clustering algorithms have a similar objective.
Thus, we use the population as a coreset weighted by the histogram and run standard $k$-means, using the resulting centroids as the selected set for the next round (Line~\ref{line:k_means_select}).\footnote{We could instead use a $k$-medians algorithm to select candidates strictly from the population satisfying a similar objective. However, we see no reason to enforce that constraint in this work, as the standard $k$-means algorithm can move us closer to the true centroids while making the selection.}
We show that this technique is much more robust to vote splitting in Section~\ref{sec:vote_split}.

\subsection{Variation API}\label{sec:variation}
After selecting the best candidates, the next step is to continue exploring the space for better centroids by creating variations of the selected candidates.
We note that we preserve the best candidates (referred to as elitism in evolutionary algorithms) in Line~\ref{line:elitisim}, so that if no variation improves, we maintain the current progress.
In some of the previous versions of PE, the \textsf{VARIATION\_API} (Line~\ref{line:variation}) would be another API call to a foundation model.
For example, \augpe~\cite{augpe} would randomly mask words and ask a large language model to fill them in with a high temperature. 
In the case of clustering, our \textsf{VARIATION\_API} perturbs the real-valued candidate centroids directly.
We find that mutation based on Lévy flights works well, as the distribution is heavier-tailed to encourage exploration.
We follow Mantegna's algorithm~\cite{Mantegna94} to generate a Lévy stable distribution parameterized by the Lévy index $\beta$ which we denote by $\text{Lévy}(\beta)$.
To sample a value $\zeta \sim \text{Lévy}(\beta)$, we first compute
\begin{equation}
    \sigma_u = \left( \frac{\Gamma(1+\beta)\,\sin\!\left(\tfrac{\pi\beta}{2}\right)}
    {\Gamma\!\left(\tfrac{1+\beta}{2}\right)\cdot \beta \cdot 2^{(\beta-1)/2}} \right)^{1/\beta},
\end{equation}
then draw independent samples $u \sim \mathcal{N}(0, \sigma_u I_\Ddim)$ and $v \sim \mathcal{N}(0, I_\Ddim)$ that we combine to get $\zeta = \frac{u}{|v|^{1/\beta}}$.
Our variation API adds i.i.d. noise from this distribution to the existing centroids, scaled by a hyperparameter $\mutationScale$ and the radius $\radius$, namely
\begin{equation}
    \textsf{VARIATION\_API}(\mu_i) = \mu_i + \mutationScale \radius \zeta
\end{equation}
where $\zeta \sim \text{Lévy}(\beta)$. 
Finally, we clip any mutated sample with $\ell_2$ norm greater than $\radius$ to the boundary.
We experimentally set the hyperparameters for our variation API and validate the use of Lévy over Gaussian noise in Section~\ref{sec:var_params}.

\subsection{Adaptive Population Size}\label{sec:adaptive_size}
Following \augpe, we apply the variation API $\numVariations$ times to each sample to increase exploration (Line~\ref{line:variation}).
Tuning the $\numVariations$ parameter is critical.
If $\numVariations$ is too small, the algorithm converges slowly, needing more iterations and privacy budget.
Alternatively, if $\numVariations$ is too large, the histogram becomes sparse, resulting in a poor signal-to-noise ratio and poor-quality selections.
While we can tune $\numVariations$ as a hyperparameter (we show a heuristic for setting it in Section~\ref{sec:hypers}), the best setting for $\numVariations$ varies greatly depending on the distribution of the private data and the current distribution of the candidates.
To help avoid the more damaging case of a large $\numVariations$ preventing meaningful selection, we add adaptive tuning of $\numVariations$ during evolution.
In Line~\ref{line:check_signal_noise}, we estimate the signal-to-noise ratio (without using any privacy budget) by computing the $\ell_2$ norm of the noisy histogram vs. the expected $\ell_2$ norm of the noise.
We argue that if this value is less than $1$, there is likely more noise than signal, and we reduce $\numVariations$ by half (Line~\ref{line:reduce_L}). 
We experimented with different thresholds and amounts of reduction but found these settings work well in practice.

\subsection{High-Dimensional Case}\label{sec:hd_pe}

\begin{algorithm}[h]
    \caption{High-Dimensional \ourPE~(HD\ourPE)}
    \label{alg:projected_PE}
    \begin{algorithmic}[1]
        \renewcommand{\algorithmicrequire}{\textbf{Input:}}
        \renewcommand{\algorithmicensure}{\textbf{Output:}}
        \Require \begin{tabular}[t]{@{}l@{}}
            Private samples: $\Dataset = \{x_i\}_{i=1}^{\DatasetSize} \subset \mathbb{R}^\Ddim$ \\
            Reduced dimension: $\Rdim$ \\
            Number of clusters: $k$ \\
            Number of PE iterations: $T$ \\
            Number of PE variations: $\numVariations$ \\
            Noise multiplier: $\noisemultiplier$ \\
            Domain radius: $\radius$
        \end{tabular}
        \Ensure DP cluster centres $S = \{\mu_1, \dots, \mu_k\} \subset \mathbb{R}^d$
        \Statex 
        
        \State Sample random projection matrix $G \sim \mathcal{N}(0,1)^{p \times \Ddim}$ \LineLabel{line:sample_G}
        \State $\bar{\Dataset} \leftarrow \{\bar{x}_i\}_{i=1}^{\DatasetSize}$ where $\bar{x}_i = \frac{1}{\sqrt{\Ddim}} G x_i$ \Comment{Project data to $\mathbb{R}^\Rdim$} \LineLabel{line:project_data}
        \State $\{\bar{\mu}_1, \dots, \bar{\mu}_k\}\leftarrow \textsf{\ourPE}(\bar{\Dataset}, T-2, k, L, \sigma)$ \Comment{Get centroids in $\mathbb{R}^\Rdim$ using Algorithm~\ref{alg:mod_PE}}\LineLabel{line:run_pe}
        
        \For{$i \leftarrow 1, \ldots, \DatasetSize$} \LineLabel{line:label_start}
            \State $a_i \leftarrow \arg\min_{j \in \{1, \dots, k\}} \|\bar{x}_i - \bar{\mu}_j\|_2$ \Comment{Assign labels in $\mathbb{R}^\Ddim$}
        \EndFor \LineLabel{line:label_end}
        
        \For{$j \leftarrow 1, \ldots, k$} \LineLabel{line:agg_start}
            \State $z_j \leftarrow \sum_{i : a_i = j} x_i + \mathcal{N}(0, \radius \noisemultiplier I_n)$ \Comment{Noisy sum in $\mathbb{R}^\Ddim$}
            \State $n_j \leftarrow \sum_{i : a_i = j} 1 + \mathcal{N}(0, \noisemultiplier I_n)$ \Comment{Noisy count}
            \State $\mu_j \leftarrow \frac{z_j}{n_j}$ 
        \EndFor \LineLabel{line:agg_end}
        
        \State \Return $S = \{\mu_1, \dots, \mu_k\}$ \LineLabel{line:return_final}
    \end{algorithmic}
\end{algorithm}

In high-dimensional domains (especially when $k$ is small), variations become less effective due to the curse of dimensionality.
To combat this, in Algorithm~\ref{alg:projected_PE}, we present a modified version of \ourPE~that is optimized for higher dimensions.
We follow the work of Balcan et al.~\cite{icml17_clustering} and first use a Johnson-Lindenstrauss transform to project the private data to a low-dimensional space (Line~\ref{line:project_data}).
We then run \ourPE~as normal in this low-dimensional space in Line~\ref{line:run_pe}.
To project the resulting cluster centres back to the original domain, we use noisy averaging. This can be thought of as a single iteration of DP Lloyd's algorithm~\cite{su_clustering} in the original domain, initialized with the low-dimensional cluster assignments.
Specifically, all private data points are assigned a label using the low-dimensional centroids (Lines~\ref{line:label_start} to \ref{line:label_end}).
Then we compute the sum and count of all private points for a given label to obtain the centroids in high-dimensional space (Lines~\ref{line:agg_start} to \ref{line:agg_end}).
To ensure privacy, we add Gaussian noise to the sum and count of points in each cluster.
We reduce the number of iterations for PE by two in Line~\ref{line:run_pe}, so we can use the same noise multiplier to noise the sum and count. 
We show in Section~\ref{sec:experiments} that HD\ourPE~outperforms \ourPE~in higher dimensions.
The disadvantage of this approach is that, similar to related work~\cite{google_clustering, icml17_clustering}, we add noise proportional to $\radius$ (the radius of the data domain) during the projection back to high dimensions. However, most other approaches apply strictly more noise, by applying radius-proportional noise in each iteration. Our approach takes advantage of the low sensitivity of \ourPE~to fit the centroids accurately in low-dimensional space before conducting a single iteration of radius-proportional noise in the final iteration.\footnote{We note that \google~also only uses a single iteration of radius-proportional noise. However, their approach assigns cluster membership based on a single LSH partition of the space, whereas we iteratively optimize cluster membership with PE, yielding a more accurate clustering.}
In future work, we will investigate alternate ways to scale the variation API while reducing this dependency.

\subsection{Privacy Analysis}\label{sec:privacy}
The privacy analysis of \ourPE~follows from the original PE paper~\cite{pe} as our changes do not affect noise addition in Algorithm~\ref{alg:voting} or use private data in any way. We restate the analysis here for completeness.
\begin{theorem}\label{thm:privacy_PE}
    \ourPE~as defined in Algorithm~\ref{alg:mod_PE} satisfies $(\epsilon,\delta)$-DP.
\end{theorem}
\begin{proof}
We break the proof into the same steps as the original paper~\cite{pe}:

\begin{itemize}
    \item \textbf{Step 1: Bounding the sensitivity (Definition~\ref{def:sensitivity}) of the \dpvotingname{} in Algorithm~\ref{alg:voting}.} %
    Each private sample only contributes one vote. If we add or remove one sample, the resulting histogram will change by 1 in the $\ell_2$ norm. Therefore, the sensitivity is 1.
    \item \textbf{Step 2: Viewing each PE iteration as a Gaussian mechanism.} In Line~\ref{line:noise_votes} of Algorithm~\ref{alg:voting}, we add i.i.d. Gaussian noise with standard deviation $\noisemultiplier$ to each bin. This is the only part of the algorithm that touches the private dataset. %
    \item \textbf{Step 3: Viewing the entire PE algorithm as $\numiterations$ adaptive compositions of Gaussian mechanisms.} This holds because PE simply applies \cref{alg:voting} $\numiterations$ times sequentially.
    \item \textbf{Step 4: Viewing the entire PE algorithm as one Gaussian mechanism with noise multiplier $\noisemultiplier/\sqrt{\numiterations}$.} Using the composition theorem (Theorem~\ref{thm:gdp_comp}) from Dong et al.~\cite[Corollary 2]{dong22_gdp}, we get that $\numiterations$ applications of a $1/\noisemultiplier$-GDP algorithm are $\sqrt{T}/\noisemultiplier$-GDP.
    \item \textbf{Step 5: Computing DP parameters $\epsilon$ and $\delta$.} The problem is simply computing a $\noisemultiplier$ such that $(\epsilon,\delta)$-DP iff $\sqrt{T}/\noisemultiplier$-GDP, for which we apply the formula of Balle and Wang~\cite{balle18_gaussian}.
\end{itemize}
\end{proof}  
Our high-dimensional extension of PE includes additional use of the private data in the projection, and thus we state and prove its privacy.
\begin{theorem}
    HD\ourPE~as defined in Algorithm~\ref{alg:projected_PE} satisfies $(\epsilon,\delta)$-DP.
\end{theorem}
\begin{proof}
~
\begin{itemize}
    \item \textbf{Step 1: Composition accounting.} We first assume that PE is called (in Line~\ref{line:run_pe}) with two fewer iterations than the number of iterations $\noisemultiplier$ was computed for.
Applying Theorem~\ref{thm:privacy_PE}, we get that the run of PE is private with two extra applications of a Gaussian mechanism with standard deviation of $\noisemultiplier$ remaining.
\item \textbf{Step 2: Assignment as post-processing.} After calling PE, the assignment step (Lines~\ref{line:label_start} to \ref{line:label_end}) uses the published centroids from the previous iteration (post-processing) to divide the dataset into clusters. We can then apply parallel composition over each of the clusters.
Thus, we can focus on the privacy cost of a single cluster for the remainder of the proof.
\item \textbf{Step 3: Privacy of the Gaussian mechanisms.} For each cluster, we apply the Gaussian mechanism twice.
The first is to compute a sum which has an $\ell_2$ sensitivity of $\radius$ since adding or removing a point can at most add the largest possible vector in the domain.
The second is to compute the count, which has an $\ell_2$ sensitivity of $1$, as a data point can be counted at most once.
Applying Theorem~\ref{thm:g_mech}, the result follows.
\end{itemize}

\end{proof}

\begin{figure*}[h!]
    \centering
    \includegraphics[width=\linewidth]{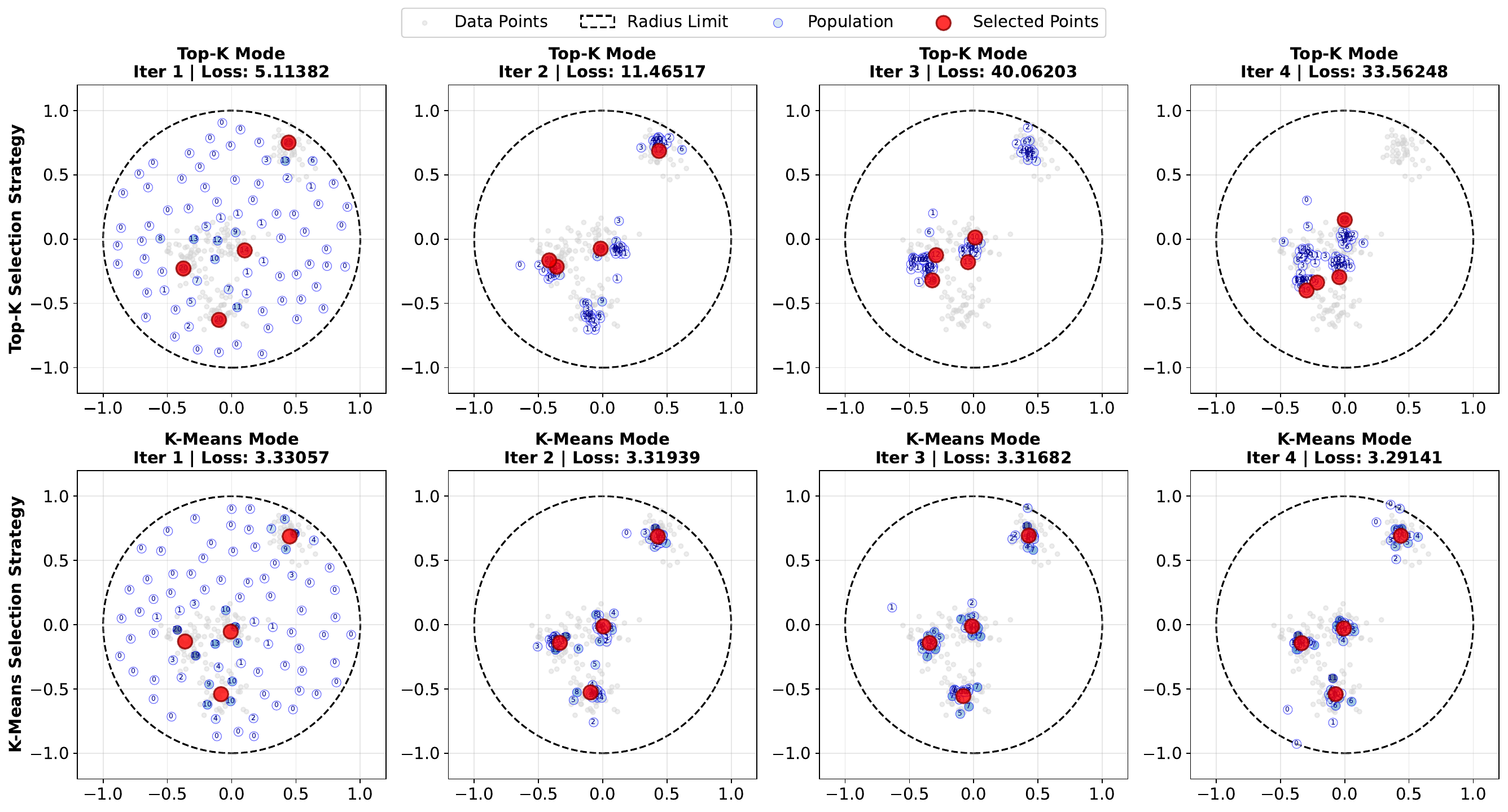}
    \caption{A visualization of the private data, samples generated, and samples selected by PE in each iteration, comparing the previous top-$k$ selection approach to our proposed $k$-means approach with $\epsilon=\infty$. The $k$-means approach effectively mitigates the vote splitting problem and keeps selected points close to the true clusters.}
    \label{fig:vote_split}
\end{figure*}

\section{Experiments}\label{sec:experiments}
In this section, we provide a complete experimental evaluation of \ourPE~and HD\ourPE. We begin by detailing the experimental setup in Section~\ref{sec:exp_setup}, followed by a case study highlighting the effect of vote splitting in practice in Section~\ref{sec:vote_split}. Section~\ref{sec:hypers} chooses and validates effective default hyperparameters for our approach. Finally, Section~\ref{sec:exp_results} gives a complete utility benchmark compared to the baselines with additional discussion in Section~\ref{sec:discussion}.
\begin{figure*}[h]
    \centering
    \includegraphics[width=0.85\linewidth]{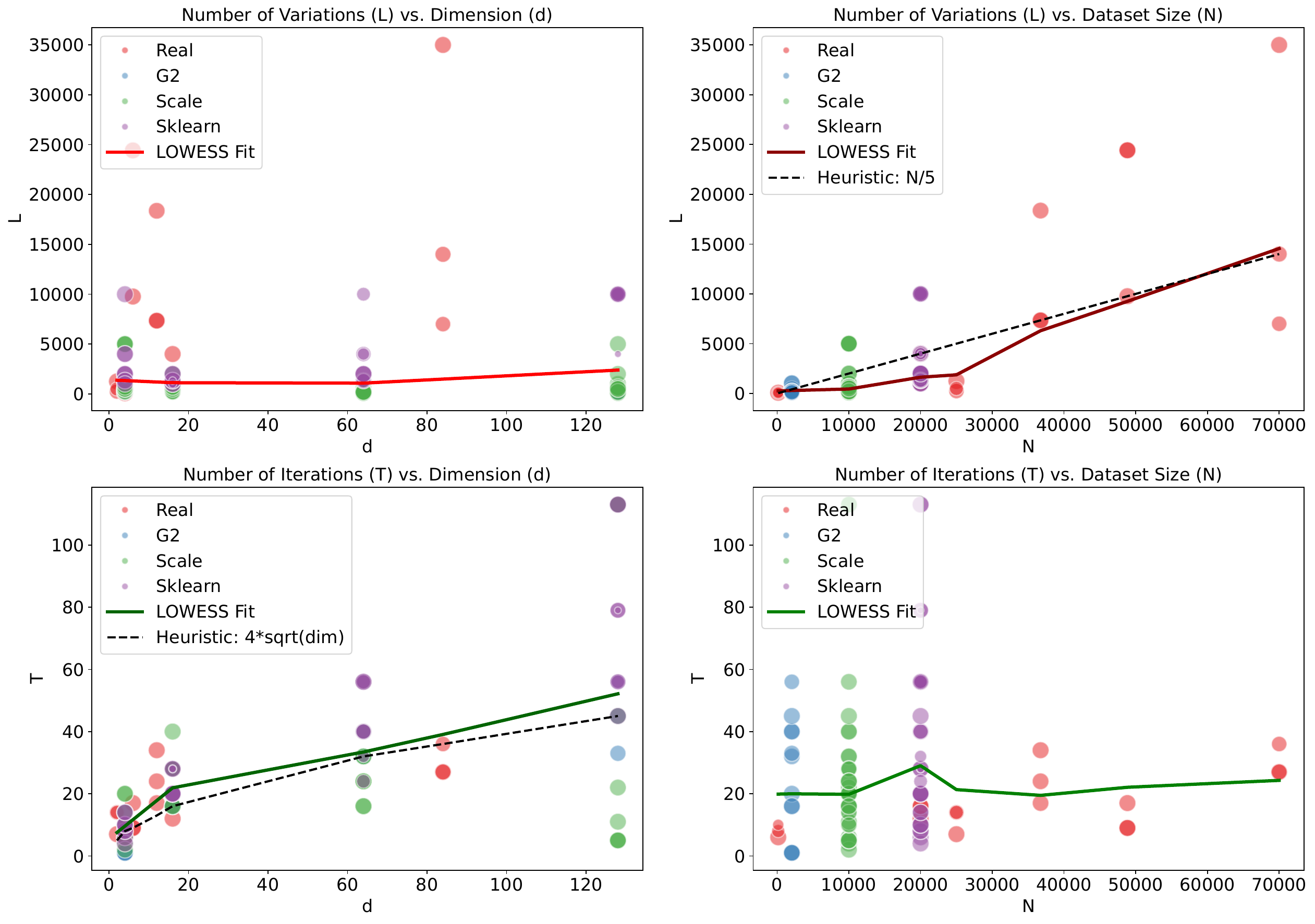}
    \caption{A scatter plot of the top-$3$ parameter configurations  of $\numVariations$ and $T$ for each dataset at $\epsilon=1$. The LOWESS line of best fit shows a clear relationship between $\numVariations$ and $\DatasetSize$ and between $T$ and $\Ddim$ that are captured by our heuristics.}
    \label{fig:hyper_tuning}
\end{figure*}
\subsection{Experimental Setup}\label{sec:exp_setup}
\paragraph{Datasets}
 We use a combination of real and synthetic datasets from various sources. Most of the real datasets as well as the G2 (Gaussian-based synthetic) datasets come from the clustering datasets repository~\cite{ClusteringDatasets}, with some additional real datasets from the UCI machine learning repository~\cite{UCI_Repo}. 
For Birch2~\cite{Birchsets}, we take $25,000$ random samples from the dataset of $100,000$.
Gas~\cite{UCI_Repo}, Letter~\cite{UCI_Repo}, and MNIST~\cite{lecun1998mnist} are included for comparability with Chang and Kamath~\cite{google_clustering}.
For MNIST, we train a LeNet5 model and use neural representations following Chang and Kamath~\cite{google_clustering}.
The \textit{scale} synthetic datasets from Diaa et al.~\cite{diaa2025fastlloyd} were generated with the \texttt{clusterGeneration} R package~\cite{clusterGen} and include a random level of cluster overlap.
Finally, we include synthetic datasets generated with the \texttt{make\_blobs} function from Sklearn that generates isotropic Gaussian clusters. Table~\ref{tab:auc_results} includes a complete list of datasets and their sizes.
\begin{figure}[h!]
    \centering
    \includegraphics[width=\linewidth]{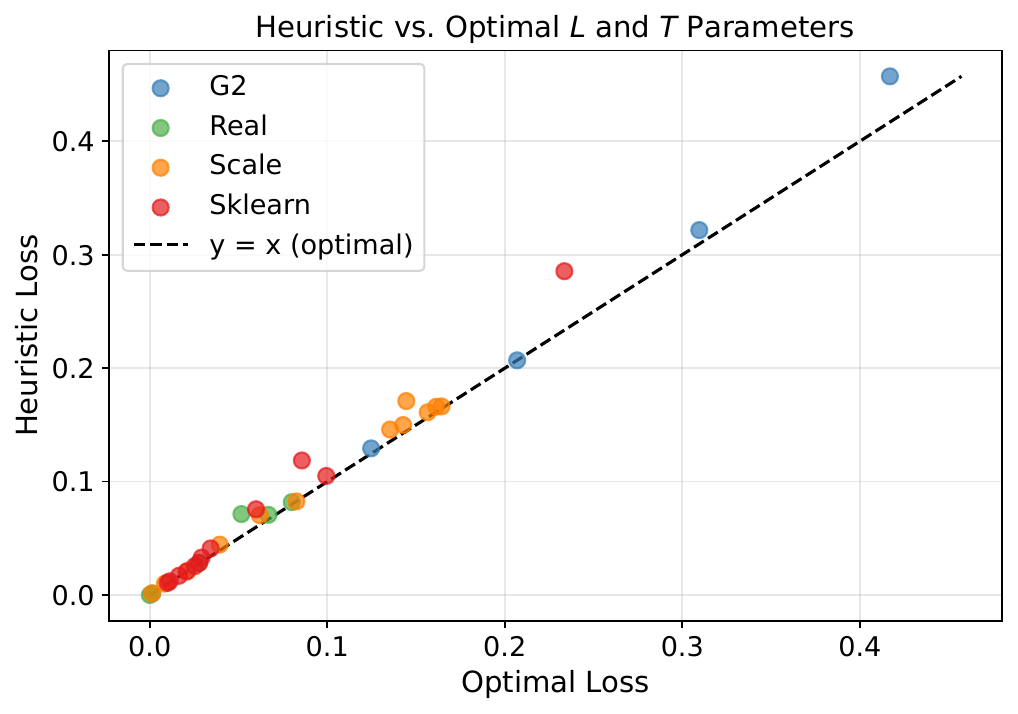}
    \caption{A scatter plot of grid search results over $\numVariations$ and $\numiterations$ with $\epsilon=1$. There is one point per dataset with the x-axis shows the clustering loss of the optimal parameter setting in the grid and the y-axis showing the loss of our heuristic for that dataset. We observe that the loss of our heuristic parameter settings are close to the optimal loss for all datasets.}
    \label{fig:LT_gap}
\end{figure}
\paragraph{Baselines}
Our baselines are selected from state-of-the-art DP $k$-means algorithms with publicly available implementations.
The work of Chang and Kamath (which we denote by \textit{\google}) is perhaps the most popular approach~\cite{google_clustering}. Chang and Kamath use a locality-sensitive hash tree to create a private coreset, upon which the non-private $k$-means algorithm is applied. We modify their code to fit our assumption that the size of the dataset is public (requiring less noise). Because the \google~method is one of the most recent works, we also evaluate the approaches from its evaluation. 
The first is IBM's differential privacy library~\cite{diffprivlib} (which we denote by \textit{\diffpriv}) that implements an improved version of Su et al.'s algorithm~\cite{su_clustering}. 
Su et al. apply noise to each step of Lloyd's clustering algorithm~\cite{lloyd1982} with various optimizations such as the sphere packing initialization. 
The second is the work of Balcan et al. (which we denote by \textit{\balcan}) that first projects the data to a low-dimensional space before recursively dividing the space into cubes. They then apply a $k$-medians-style swapping algorithm to choose the best centres before projecting the result back into the high-dimensional space through noisy averaging. We discovered several inconsistencies in the privacy analysis of Balcan et al.'s implementation, which we detail in the documentation of our implementation. 
We note that both \diffpriv~and \balcan~work in the pure differential privacy model and thus, following \google~\cite{google_clustering}, we evaluate them with $\delta=0$.
The final related work is a recent improvement over Su et al.'s algorithm by Diaa et al.~\cite{diaa2025fastlloyd}, which we denote by \textit{\fastlloyd}. \fastlloyd~modifies Su et al.'s work by using Gaussian noise and computing cluster updates relative to the previous iteration, thereby reducing the sensitivity.
\begin{figure}[t]
    \centering
    \includegraphics[width=\linewidth]{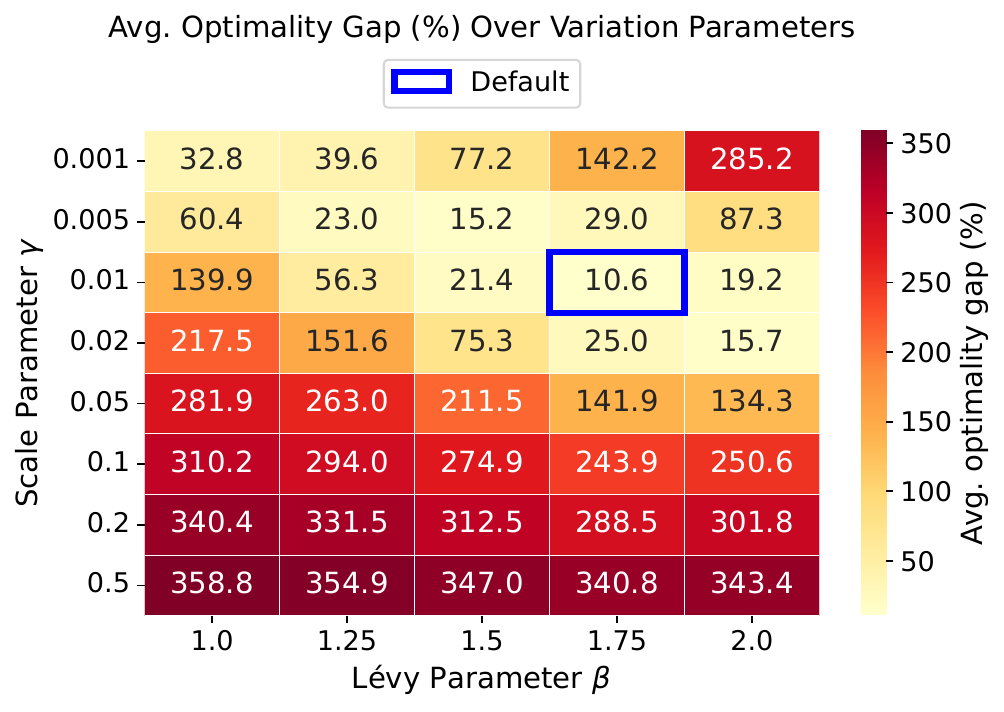}
    \caption{A heat map of the grid search results over $\mutationScale$ and $\beta$ at $\epsilon=1$ with each cell showing the percentage increase in clustering loss for a given hyperparameters setting compared to the optimal hyperparamter choice for each dataset, averaged over all datasets. We observe our default parameter choice is the closest to the optimal loss.}
    \label{fig:mut_heatmap}
\end{figure}
\begin{figure}[t]
    \centering
    \includegraphics[width=\linewidth]{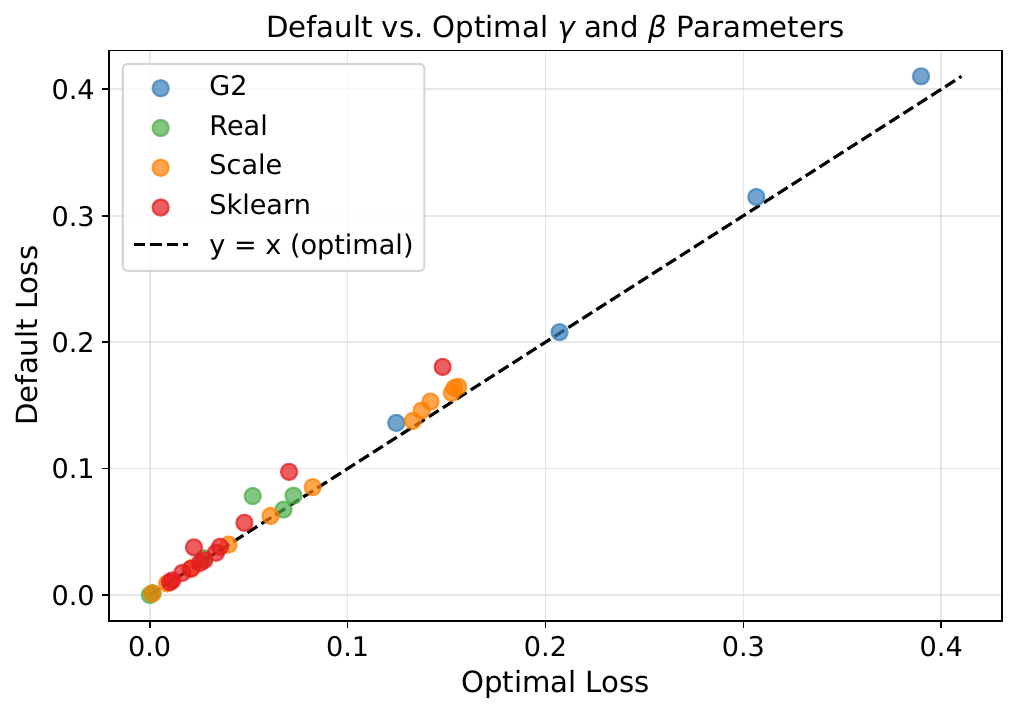}
    \caption{A scatter plot of grid search results over $\mutationScale$ and $\beta$ with $\epsilon=1$. There is one point per dataset with the x-axis shows the clustering loss of the optimal parameter setting in the grid and the y-axis showing the loss of our default parameter values. We observe that the loss of our default parameter settings are close to the optimal loss for all datasets.}
    \label{fig:mut_scatter}
\end{figure}
\afterpage{
\begin{table*}[ht]
\centering
\setlength{\tabcolsep}{4pt}       %
\renewcommand{\arraystretch}{0.9} %
\caption{A summary of results over all datasets (real datasets highlighted in bold) showing the AUC of the clustering loss across $\epsilon$ values with the best approach highlighted for each dataset. The final column shows the relative improvements of PE-based techniques over the state-of-the art baseline for that dataset, with an average improvement of $26\%$ over all datasets for PE-based approaches.}
\label{tab:auc_results}
\resizebox{\linewidth}{!}{%
\begin{tabular}{lrrr|r|rrrrrr|r}
\toprule
Dataset & $N$ & $d$ & $k$ & Non-Priv & PE-means & HDPE-means & \google & \fastlloyd & \diffpriv & \balcan & PE Improv. \\
\midrule
\textbf{birch2} & 25000 & 2 & 100 & 0.0001 & \bfseries {\cellcolor[HTML]{D4EDDA}} 0.0003 & - & 0.0658 & 0.0031 & 0.0124 & 0.0161 & +90.92\% \\
\textbf{iris} & 150 & 4 & 3 & 0.1361 & \bfseries {\cellcolor[HTML]{D4EDDA}} 0.2894 & - & 0.6892 & 0.3979 & 1.3048 & 5.5766 & +27.28\% \\
\textbf{adult} & 48842 & 6 & 3 & 0.0056 & \bfseries {\cellcolor[HTML]{D4EDDA}} 0.0056 & - & 0.0076 & 0.0155 & 0.0213 & 0.0113 & +25.55\% \\
\textbf{mnist} & 70000 & 84 & 10 & 0.1958 & 0.2920 & \bfseries {\cellcolor[HTML]{D4EDDA}} 0.2169 & 0.2173 & 0.5822 & 0.5315 & 0.7904 & +0.16\% \\
\textbf{letter} & 20000 & 16 & 26 & 0.2481 & \bfseries {\cellcolor[HTML]{D4EDDA}} 0.2852 & - & 0.3187 & 0.3963 & 0.4341 & 0.6277 & +10.50\% \\
\textbf{gas} & 36733 & 12 & 6 & 0.1026 & \bfseries {\cellcolor[HTML]{D4EDDA}} 0.1081 & - & 0.1121 & 0.1561 & 0.1761 & 0.1853 & +3.60\% \\
g2\_4 & 2048 & 4 & 2 & 0.4650 & 0.4761 & - & 0.5465 & \bfseries {\cellcolor[HTML]{D4EDDA}} 0.4666 & 0.5468 & 0.8802 & -2.05\% \\
g2\_16 & 2048 & 16 & 2 & 0.7677 & \bfseries {\cellcolor[HTML]{D4EDDA}} 0.7800 & - & 0.8391 & 0.8687 & 1.1816 & 4.7489 & +7.04\% \\
g2\_64 & 2048 & 64 & 2 & 1.0956 & 1.1538 & \bfseries {\cellcolor[HTML]{D4EDDA}} 1.1295 & 1.2478 & 1.9647 & 2.6608 & 27.5361 & +9.49\% \\
g2\_128 & 2048 & 128 & 2 & 1.3534 & 1.7237 & \bfseries {\cellcolor[HTML]{D4EDDA}} 1.4055 & 1.5906 & 2.6425 & 4.4632 & 74.0978 & +11.64\% \\
scale\_4\_4 & 10000 & 4 & 4 & 0.0821 & \bfseries {\cellcolor[HTML]{D4EDDA}} 0.0814 & - & 0.1190 & 0.1272 & 0.1381 & 0.1772 & +31.54\% \\
scale\_4\_16 & 10000 & 16 & 4 & 0.3013 & \bfseries {\cellcolor[HTML]{D4EDDA}} 0.3147 & - & 0.3392 & 0.3721 & 0.4114 & 0.5948 & +7.24\% \\
scale\_4\_64 & 10000 & 64 & 4 & 0.4966 & \bfseries {\cellcolor[HTML]{D4EDDA}} 0.5326 & 0.5340 & 0.5430 & 0.5863 & 0.6622 & 2.5354 & +1.92\% \\
scale\_4\_128 & 9999 & 128 & 4 & 0.5731 & 0.6179 & \bfseries {\cellcolor[HTML]{D4EDDA}} 0.5940 & 0.5987 & 0.6282 & 0.7481 & 8.5885 & +0.79\% \\
scale\_16\_4 & 10000 & 4 & 16 & 0.0302 & \bfseries {\cellcolor[HTML]{D4EDDA}} 0.0335 & - & 0.0555 & 0.0864 & 0.0976 & 0.1525 & +39.63\% \\
scale\_16\_16 & 10000 & 16 & 16 & 0.2110 & \bfseries {\cellcolor[HTML]{D4EDDA}} 0.2452 & - & 0.2880 & 0.3624 & 0.3778 & 0.5679 & +14.87\% \\
scale\_16\_64 & 10000 & 64 & 16 & 0.4824 & 0.5884 & \bfseries {\cellcolor[HTML]{D4EDDA}} 0.5757 & 0.5823 & 0.6371 & 0.6949 & 2.5353 & +1.14\% \\
scale\_16\_128 & 10001 & 128 & 16 & 0.5410 & 0.6131 & 0.6023 & \bfseries {\cellcolor[HTML]{D4EDDA}} 0.5992 & 0.6294 & 0.7612 & 8.6815 & -0.50\% \\
scale\_64\_4 & 10002 & 4 & 64 & 0.0029 & \bfseries {\cellcolor[HTML]{D4EDDA}} 0.0051 & - & 0.0214 & 0.0306 & 0.0381 & 0.0847 & +76.08\% \\
scale\_64\_16 & 10001 & 16 & 64 & 0.1126 & \bfseries {\cellcolor[HTML]{D4EDDA}} 0.1607 & - & 0.1902 & 0.2393 & 0.2771 & 0.5238 & +15.54\% \\
scale\_64\_64 & 10001 & 64 & 64 & 0.4487 & 0.5518 & 0.5567 & \bfseries {\cellcolor[HTML]{D4EDDA}} 0.5453 & 0.5827 & 0.6913 & 2.5620 & -1.20\% \\
scale\_64\_128 & 10016 & 128 & 64 & 0.5206 & 0.6874 & \bfseries {\cellcolor[HTML]{D4EDDA}} 0.6232 & 0.6299 & 0.6421 & 1.0248 & 8.6433 & +1.07\% \\
sklearn\_4\_4 & 20000 & 4 & 4 & 0.0777 & \bfseries {\cellcolor[HTML]{D4EDDA}} 0.0778 & - & 0.1305 & 0.1777 & 0.2355 & 0.1240 & +37.22\% \\
sklearn\_16\_4 & 20000 & 4 & 16 & 0.0459 & \bfseries {\cellcolor[HTML]{D4EDDA}} 0.0445 & - & 0.0979 & 0.0990 & 0.1409 & 0.2601 & +54.59\% \\
sklearn\_64\_4 & 20000 & 4 & 64 & 0.0343 & \bfseries {\cellcolor[HTML]{D4EDDA}} 0.0379 & - & 0.1289 & 0.0818 & 0.1191 & 0.3227 & +53.68\% \\
sklearn\_4\_16 & 20000 & 16 & 4 & 0.0950 & \bfseries {\cellcolor[HTML]{D4EDDA}} 0.0955 & - & 0.1395 & 0.7110 & 0.5049 & 0.6074 & +31.50\% \\
sklearn\_16\_16 & 20000 & 16 & 16 & 0.0606 & \bfseries {\cellcolor[HTML]{D4EDDA}} 0.0668 & - & 0.1596 & 0.7160 & 0.5809 & 1.4670 & +58.16\% \\
sklearn\_64\_16 & 20000 & 16 & 64 & 0.0642 & \bfseries {\cellcolor[HTML]{D4EDDA}} 0.1183 & - & 0.3872 & 0.7631 & 0.8720 & 1.7704 & +69.43\% \\
sklearn\_4\_64 & 20000 & 64 & 4 & 0.0954 & 0.0993 & \bfseries {\cellcolor[HTML]{D4EDDA}} 0.0974 & 0.1090 & 1.4623 & 0.8614 & 2.7731 & +10.60\% \\
sklearn\_16\_64 & 20000 & 64 & 16 & 0.0847 & 0.1192 & \bfseries {\cellcolor[HTML]{D4EDDA}} 0.1189 & 0.2258 & 1.7475 & 1.3716 & 3.1765 & +47.34\% \\
sklearn\_64\_64 & 20000 & 64 & 64 & 0.0823 & \bfseries {\cellcolor[HTML]{D4EDDA}} 0.3070 & 0.4678 & 0.7878 & 2.2002 & 2.2456 & 3.1698 & +61.03\% \\
sklearn\_4\_128 & 20000 & 128 & 4 & 0.1054 & 0.1333 & \bfseries {\cellcolor[HTML]{D4EDDA}} 0.1097 & 0.1186 & 1.7604 & 1.1632 & 5.6663 & +7.49\% \\
sklearn\_16\_128 & 20000 & 128 & 16 & 0.0919 & 0.1942 & \bfseries {\cellcolor[HTML]{D4EDDA}} 0.1562 & 0.2838 & 2.3038 & 1.9742 & 4.9881 & +44.97\% \\
sklearn\_64\_128 & 20000 & 128 & 64 & 0.0933 & \bfseries {\cellcolor[HTML]{D4EDDA}} 0.5740 & 0.7190 & 0.9675 & 2.7204 & 2.9608 & 4.9597 & +40.67\% \\
\bottomrule
\end{tabular}}
\end{table*}
}

\paragraph{Implementation Details}
Our source code is available on GitHub.\footnote{\url{https://github.com/t3humphries/PE-means}} We normalize all datasets to have a max $\ell_2$ norm of $1$, for a clean presentation of metrics over different datasets.
We follow \google~and centre each dataset first by subtracting the mean~\cite{google_clustering}. 
Then, also following \google, we non-privately compute the $\ell_2$ and $\ell_\infty$ sensitivity bounds on each dataset. 
In practice, this bound should be computed privately or derived from public information about the attributes.
However, for the sake of comparison, we give all approaches an equal advantage by providing a tight bound on the domain.
Each experiment is repeated $50$ times over different random seeds, and we report the average. 
All shaded areas represent the 95\% confidence interval of the mean of the results.
We fix the privacy parameter $\delta=1/\DatasetSize^{1.1}$, following the best practice that the failure probability should be less than $1/\DatasetSize$~\cite{dwork14}.
The primary loss we consider is the normalized $k$-means loss.
\begin{equation}
    \text{Loss} = \frac{1}{\DatasetSize} \sum_{i=1}^{\DatasetSize} \min_{j=1}^{k} \|x_i - \mu_j\|^2
\end{equation}
We also compute each method's performance over different privacy budgets using the area under the curve (AUC) of the loss values against $\epsilon \in \{0.25, 0.5, 1.0, 2.0, 4.0\}$, computed via the trapezoidal rule:
\begin{equation}
    \text{AUC} = \sum_{i=1}^{n-1} \frac{\text{Loss}_i + \text{Loss}_{i+1}}{2} \cdot (\epsilon_{i+1} - \epsilon_i)
\end{equation}

\subsection{Vote Splitting Example}\label{sec:vote_split}
In Section~\ref{sec:selection}, we described the issue of vote splitting in PE's selection. In this section, we illustrate the phenomenon with a toy example. We generate a two-dimensional random dataset with $k=4$ using the \texttt{make\_blobs} function and set $\epsilon=\infty$. In Figure~\ref{fig:vote_split}, we plot four iterations of PE (one per column) and show the difference between the previous top-$k$ voting (top row of plots) and our weighted $k$-means selection (bottom row of plots). The plots include the private data in grey, the population of PE in blue, and the selected points for each method in red. 

In the case of the top-$k$ voting (first row), we observe that the initial iteration chooses points close to the four true clusters in the dataset. However, in the second iteration, there are many good candidates in the lower cluster, splitting the vote and causing all selected points to be in the other three clusters. Then, in the third iteration, the lower cluster dies out and the vote becomes split in the upper cluster. The clusters do not recover in the fourth iteration, as it will now take PE numerous iterations to move back towards those areas. In contrast, the bottom row of plots shows PE maintaining and refining four healthy clusters throughout the iterations. We note that in this toy example, PE has essentially converged in the first iteration, which further highlights the challenge with top-$k$ voting. 

\begin{figure*}[h]
    \centering
    \includegraphics[width=\linewidth]{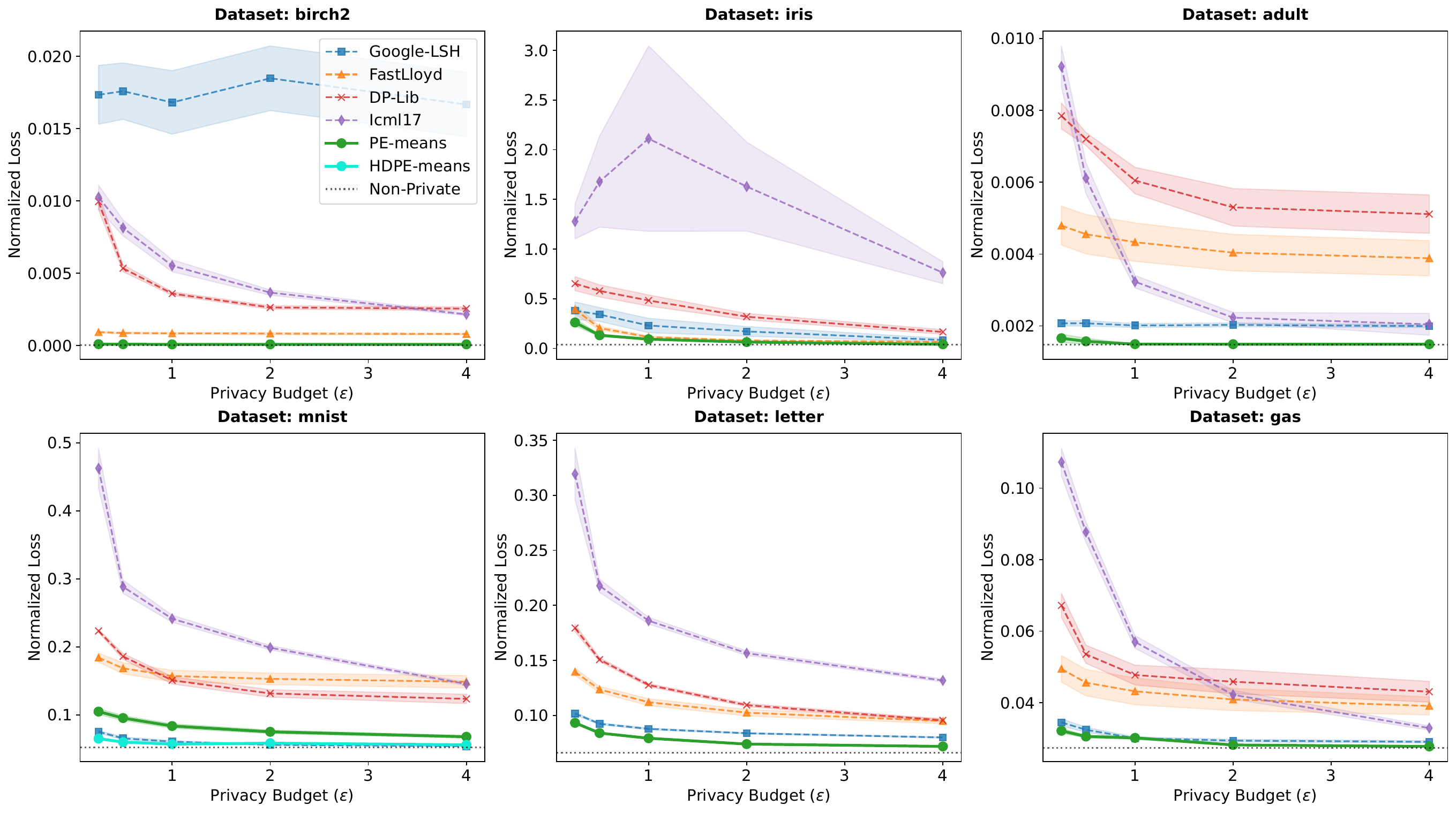}
    \caption{A plot of the privacy-utility trade-off of all approaches on real datasets, showing a PE-based approach always performs best.}
    \label{fig:real_sets}
\end{figure*}

\begin{figure*}[h]
    \centering
    \includegraphics[width=\linewidth]{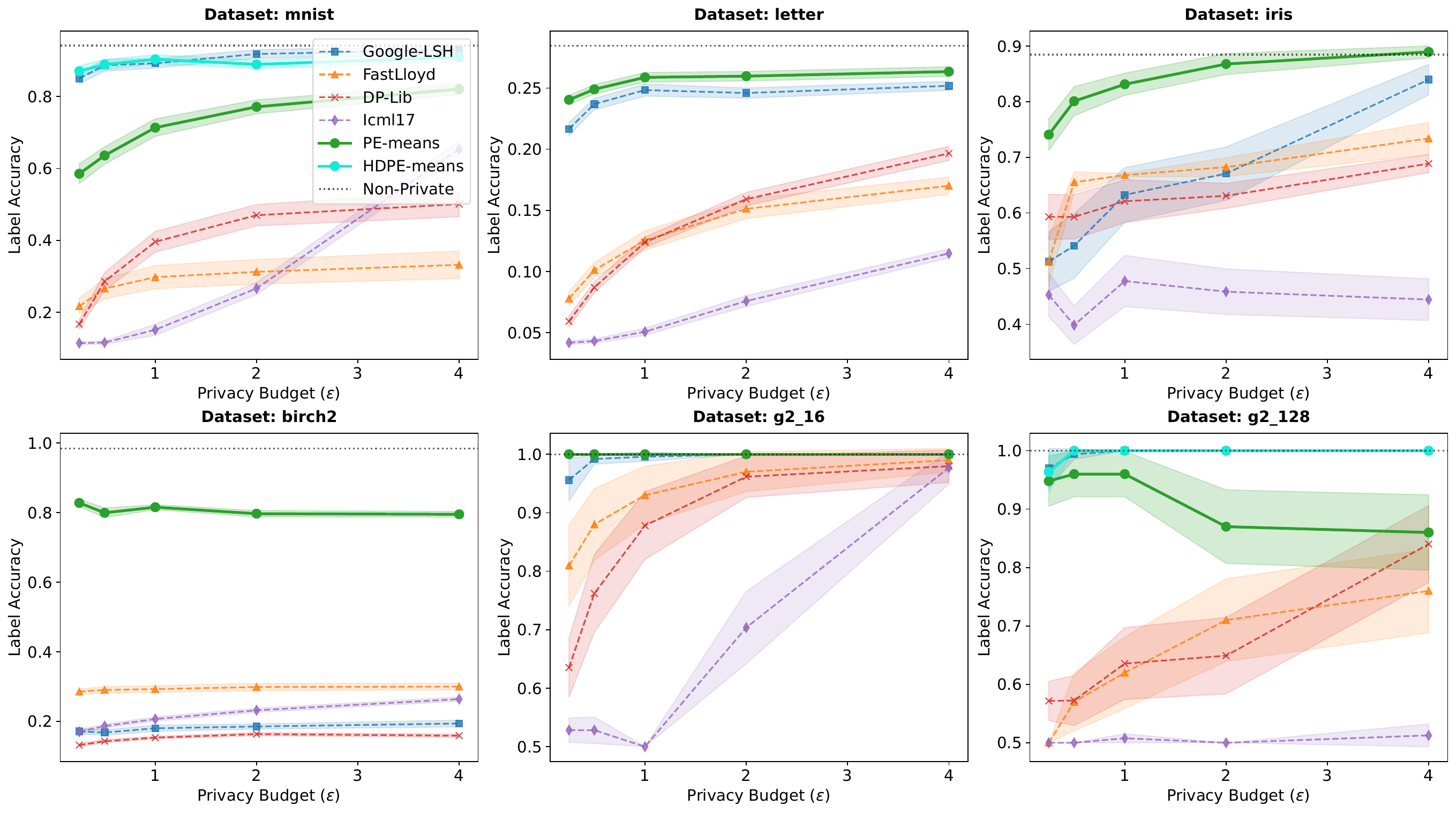}
    \caption{A plot of the label accuracy metric vs. privacy budget $\epsilon$ over all approaches on a subset of real datasets that have ground-truth labels. We observe that a PE-based approach always performs better or similar to state-of-the-art.}
    \label{fig:label_acc}
\end{figure*}

\subsection{Hyperparameter Tuning}\label{sec:hypers}

\subsubsection{Iterations and Variation Size}
Two of the most influential hyperparameters for \ourPE~are the number of iterations $\numiterations$ and the number of variations $\numVariations$. For a fixed privacy budget, if the number of iterations is too high, the amount of noise used in each iteration grows too large, preventing meaningful selections. If the algorithm executes too few iterations, it will not converge. Similarly, if $\numVariations$ is too large, the histogram can become sparse and the votes will not dominate even the smallest amount of noise. If $\numVariations$ is too small, then the chance of \ourPE~finding points closer to the private data is reduced, slowing convergence. We recall that our adaptive reduction of $\numVariations$ (Line~\ref{line:reduce_L} of Algorithm~\ref{alg:mod_PE}) protects against the case when $\numVariations$ is causing the noise to dominate the signal, but it is still beneficial to set a good starting point.

The only public parameters we can use to estimate the difficulty of a dataset in the private setting are its dimensions. In our initial testing, we observed that the number of iterations is influenced by the dimension of the dataset. Namely, in higher-dimensional space, PE needs more steps due to the curse of dimensionality. We also observed that the sparsity of the histogram for a given $\numVariations$ is highly influenced by the number of values in a dataset that can vote. To investigate this more rigorously, we conduct an experiment where we vary $\numVariations$ and $\numiterations$ for each dataset and plot the values against the dataset size and dimension.
This experiment uses all datasets to give a variety of datasets spanning dimensions, dataset sizes, and values of $k$. We plot the results in Figure~\ref{fig:hyper_tuning}, where $\epsilon=1$. We scatter the values of the top-$3$ best-performing parameters for each dataset and draw a line of best fit using the LOWESS method. While the top-left plot reveals no useful relationship between the number of variations $\numVariations$ and the number of dimensions $\Ddim$, the top-right plot shows a clear relationship between the dataset size $\DatasetSize$ and $\numVariations$. We approximate this relationship with the heuristic $\numVariations=max(\DatasetSize/5,4)$. 
Similarly, we see a clear relationship between the dimension $\Ddim$ and number of iterations $\numiterations$ in the bottom-left plot, which we approximate with the heuristic $\numiterations=max(4\sqrt{\Ddim},1)$. Finally, we see no clear relationship between $\DatasetSize$ and $\numiterations$ in the bottom-right plot. In our experiments, we never encounter the constant cases in the max terms, but include them to ensure a robust implementation.
\begin{figure*}[h]
    \centering
    \vspace{4pt}
    \includegraphics[width=\linewidth]{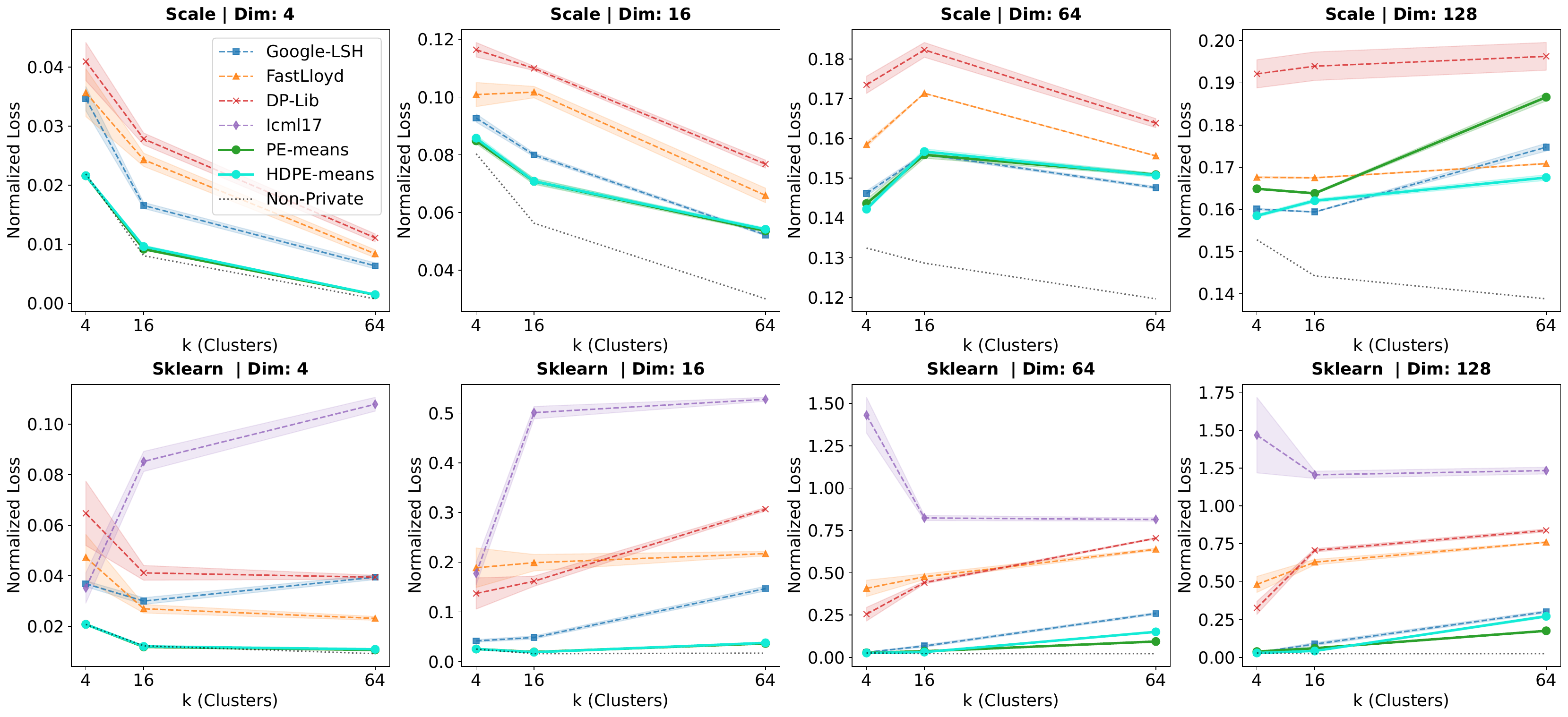}
    \caption{A comparison illustrating how approaches scale with the number of clusters $k$ over different synthetic datasets at $\epsilon$=1.0.}
    \label{fig:synth_sets}
\end{figure*}
While Figure~\ref{fig:hyper_tuning} demonstrates that our heuristics closely follow the line of best fit, it is also clear that the optimal settings for a number of datasets are quite far away from the line.
To further investigate this, we consider the difference in utility between the optimal settings and the heuristic settings in Figure~\ref{fig:LT_gap}.
The x-axis represents the loss of the optimal configuration of $\numiterations$ and $\numVariations$ over the entire grid search, and the y-axis represents the loss when $\numiterations$ and $\numVariations$ are set following the heuristic.
We can see most points lie on or very close to the line $y=x$, indicating near-optimality. 
Thus, we argue our heuristics offer good default values in practice.

To extend these heuristics to also account for the privacy budget $\epsilon$, we experimented with scaling them by factors such as $\epsilon, 0.5\epsilon, 2\epsilon, \sqrt{\epsilon}$ (since both parameters should decrease with smaller epsilon). We find that for values of $\epsilon<1$, no modification improves over the baseline, as our adaptive strategy increases the signal-to-noise ratio for any dataset that needs it (and not all do). However, when $\epsilon>1$, we find it is best (taking the median performance over all datasets) to scale the number of iterations by $\epsilon$ ($\numiterations=max(4\epsilon\sqrt{\Ddim},1)$).

\subsubsection{Variation Parameters}\label{sec:var_params}
The remaining hyperparameters are those of our variation API. The first is the Lévy index $\beta$ that determines how heavy the tails of the distribution are. Smaller values indicate a heavy tail with $\beta=1$ representing a Cauchy distribution, and $\beta=2$ representing a Gaussian. In our implementation, we add a condition to explicitly draw from a standard Gaussian at $\beta=2$.
The second parameter, the variation scale $\mutationScale$, is a multiplicative factor to control how much to vary the centroids. We conduct a grid search over possible values of each parameter (with $\numVariations$ and $\numiterations$ fixed to the heuristics from the previous section) and plot the results in Figure~\ref{fig:mut_heatmap}. 
We plot the grid search results as a heatmap with each cell showing the percentage increase of this parameter configuration over the optimal parameter configuration, averaged over all datasets we evaluate (see Table~\ref{tab:auc_results} for the complete list of datasets). We can see that $\beta=1.75$ and $\mutationScale=0.01$ give the best results and thus, we choose these values as our default. In Figure~\ref{fig:mut_scatter}, we scatter the optimal vs. default values of $\beta$ and $\mutationScale$ over all datasets. We can see that the default parameters stay close to the optimal line of $y=x$, indicating a good choice of parameters.

\subsection{Utility Benchmark}\label{sec:exp_results}
\paragraph{Overall Comparison}
We compare \ourPE~and HD\ourPE~to the baselines over all datasets and summarize the results in Table~\ref{tab:auc_results}. The metric we use is the AUC of the $k$-means loss (a summary of performance over a collection of epsilons, with smaller values being best).\footnote{To convert the non-private $k$-means loss to an AUC score, we compute the area under the horizontal line drawn at the non-private utility over the same set of epsilons evaluated for the private approaches.} We highlight the approach with the lowest loss AUC in green. HD\ourPE~is only included when the dimension of the dataset is greater than $16$, as that is the dimensionality we project down to. 
Overall, \ourPE~and HD\ourPE~perform the best over most datasets, with \fastlloyd~doing well on the g2\_4 dataset, and \google~doing well in two of the high-dimensional scale datasets. We give the percentage of improvement (or decline) of the best PE-based approach compared to the best non-PE private approach in the last column. We see that when PE-based approaches are outperformed, it is by at most $2\%$, but on average, they improve by $26\%$, with improvements of as much as $91\%$. Finally, we note that HD\ourPE~typically outperforms \ourPE~on datasets with dimension larger than $16$. In high-dimensional cases where \ourPE~outperforms HD\ourPE, it is typically not by a significant amount compared to the performance of related work, and thus we recommend always using HD\ourPE~in $\Ddim>16$ datasets.

\paragraph{Real Dataset Privacy vs. Utility Trade-off}
We give the full privacy vs. utility trade-off for the various real datasets (including the three evaluated by \google~\cite{google_clustering}) in Figure~\ref{fig:real_sets}. We observe that a PE-based approach is always the best approach, being the closest to the non-private baseline in all plots. Among the baselines, we observe that \fastlloyd~performs well in the lower dimensions (birch2 and iris) and \google~does well in higher dimensions. 
We also note that our PE-based approaches scale well to limited privacy budgets.

\paragraph{Ground Truth Evaluation}
Following \google~\cite{google_clustering}, we also evaluate an alternate metric called cluster label accuracy. We only evaluate this metric for datasets for which we have a ground-truth labelling, namely: mnist, letter, iris, birch2, and the G2 datasets. Given a clustering of the private data, we compute the metric by first assigning each centroid a label through a majority vote of the ground-truth label of the points in that cluster. Then we simply compute accuracy based on the number of private points correctly labelled by the centroid label. The results are given in Figure~\ref{fig:label_acc}. We observe that for the letter, iris, and birch2 datasets, \ourPE~outperforms related work by a statistically significant amount as the confidence intervals do not overlap. In the remaining datasets, \ourPE~or HD\ourPE~perform similarly to the best baselines as well as the non-private baseline, indicating near-optimal performance.

\paragraph{Scaling Evaluation} In Figure~\ref{fig:synth_sets}, we show how the algorithms scale with the number of clusters and a fixed privacy budget of $\epsilon=1$. We omit \balcan~from the scale dataset plots for clarity as the loss is an order of magnitude worse than other approaches. The scale datasets follow the trend of the non-private version until we increase dimensions, at which point the noise has more of an effect at higher numbers of clusters. This is likely because the data size remains fixed, causing the number of samples voting (or being aggregated) per cluster to decline. The scale datasets also highlight the room for improvement in PE-based approaches in higher dimensions. For the Sklearn sets, which tend to be more well-separated Gaussian blobs, we observe PE-based approaches outperforming all other approaches and also scaling better with $k$.

\subsection{Discussion}\label{sec:discussion}
In addition to showing the improvements of PE-based approaches in clustering, our experimental evaluations also provide support for the PE algorithm in general.
Specifically, \ourPE~does not use foundation models to implement the random and variation APIs.
This means we can rule out the effect of dataset contamination (evaluation datasets appearing in the foundation model's training set) in our results.
This gives additional evidence to support the effectiveness of the PE paradigm independently of the improvement of foundation models.

\section{Related Work}\label{sec:related_work}
\paragraph*{DP $k$-means} In addition to the baselines we compare against in Section~\ref{sec:experiments}, there are several previous works that also solve the DP $k$-means clustering problem, which are either improved upon by our baselines or are more theoretical in nature.
One line of work focused on developing a private version of Lloyd's algorithm~\cite{blum2005, mcsherry2009, Dwork2011, su_clustering, diaa2025fastlloyd} starting with Blum et al.~\cite{blum2005} and concluding with our baselines of Su et al.~\cite{su_clustering} and FastLloyd~\cite{diaa2025fastlloyd}.
Another line of work used the sample and aggregate framework~\cite{Nissim2007,Mohan2012} to solve $k$-means, but was outperformed by Su et al.~\cite{su_clustering}. 
A theoretical line of work focused on minimizing the bounds on approximation error, but did not provide experimental evaluation~\cite{Feldman2009,Feldman2017,Nissim2018,stemmer2018,Ghazi2020,Jones_2021}.

\paragraph*{DP Evolutionary Algorithms} There has been previous work that has applied evolutionary techniques to the problem of clustering.
The first differentially private evolutionary algorithm was developed by Zhang et al.~\cite{zhang_2013}, who evaluated it on the $k$-means clustering problem.
Follow-up work by Humphries and Kerschbaum~\cite{humphries23dpga} showed that Zhang et al.'s solution suffered from prohibitively poor utility and provided an improved algorithm which was evaluated on the $k$-medians clustering problem.
While our work also applies an evolutionary-based approach to solve clustering, our algorithm uses significantly less privacy budget.
The population of \ourPE~is a set of centroids from which a single solution of $k$ points is chosen in each iteration using a DP histogram with sensitivity $1$, following the previous work in PE~\cite{pe, augpe, pe_three,tran2026differentially}.
Humphries and Kerschbaum instead evolve a population where each candidate is a solution of $k$ points and apply the exponential mechanism on the clustering loss (with sensitivity proportional to the data domain) to select multiple candidates in each iteration, using significantly more privacy budget.

\section{Conclusion}\label{sec:conclusion}
We have shown that the PE algorithm is well-suited to the problem of $k$-means clustering.
Our benchmark shows that \ourPE~and HD\ourPE~offer state-of-the-art performance on many synthetic and real clustering datasets.
Furthermore, the design of \ourPE~is of independent interest, offering solutions to vote-splitting and signal-to-noise ratio management that can be extended to other modalities of PE. In future work, we will study variation APIs that more efficiently search high-dimensional Euclidean spaces.

\begin{acks}
We would like to thank Sivakanth Gopi for helpful discussions on an earlier version of this work, including the suggestion to use MLE post-processing on the vote histogram.
\end{acks}

\bibliographystyle{ACM-Reference-Format}
\bibliography{refs}

\end{document}